%% file: main_iclr.tex
\documentclass{article} 
\usepackage{iclr2021/iclr2021_conference,times}
\usepackage{mathtools}
\usepackage{enumerate}
\input{hayden_preamble}

\newtheorem{theorem}{Theorem}[section]

\usepackage{hyperref}
\usepackage{url}

\title{A partition-based similarity for \\ classification distributions}

\newcommand\blfootnote[1]{%
  \begingroup
  \renewcommand\thefootnote{}\footnote{#1}%
  \addtocounter{footnote}{-1}%
  \endgroup
}

\author{Hayden~S.~Helm$^{*}$ \& Ronak D. Mehta \\
Microsoft Research \& Johns Hopkins University \\
\texttt{\{haydenshelm, rmehta004\}@gmail.com} \\
\AND
Brandon~Duderstadt \\
Square \\ \texttt{bduderstadt@squareup.com}
\\
\AND
Weiwei~Yang \& Christopher~White \\
Microsoft Research \\
\texttt{\{weiwya, chwh\}@microsoft.com} \\
\AND
Ali~Geisa \& Joshua~T.~Vogelstein \& Carey~E.~Priebe \\
Johns Hopkins University \\
\texttt{\{ageisa1, jovo, cep\}@jhu.edu}
}

\iclrfinalcopy 
\begin{document}
\maketitle

\blfootnote{
$ ^{*} $ corresponding author
}
\input{text/abstract.tex}






\input{text/introduction.tex}

\label{introduction}

\input{text/background.tex}
\label{background}

\input{text/methods.tex}
\label{method}

\input{text/simulations.tex}
\label{simulations}


\section{Discussion}
\input{text/discussion.tex}
\label{discussion}

\clearpage
\bibliographystyle{iclr2021/iclr2021_conference}
\bibliography{biblio}

\clearpage

\appendix
\input{text/appendix/theorems}
\input{text/appendix/illustrative}
\input{text/appendix/experiments}
\input{text/appendix/network}
\label{appendix}


\end{document}

%% file: hayden_preamble.tex
\usepackage{caption}
\usepackage{amsthm}

\newtheorem{Def}{Definition}[section]

\usepackage{subcaption}
\usepackage{graphicx}

\usepackage{natbib}

\usepackage{amsfonts,amsmath,amssymb}
\usepackage{bbm}
\usepackage{xspace}

\usepackage{multicol}
\usepackage{enumitem}

\newcommand{\iid}{\overset{iid}{\sim}}

\providecommand{\mc}[1]{\mathcal{#1}}

\providecommand{\mbb}[1]{\mathbb{#1}}

\newcommand{\argmax}{\operatornamewithlimits{argmax}}

%% file: text/abstract.tex
\begin{abstract}
Herein we define a measure of similarity between classification distributions that is both principled from the perspective of statistical pattern recognition and useful from the perspective of machine learning practitioners. 
In particular, we propose a novel similarity on classification distributions, dubbed \textit{task similarity}, that quantifies how an optimally-transformed optimal representation for a source distribution performs when applied to inference related to a target distribution. 
The definition of task similarity allows for natural definitions of adversarial and orthogonal distributions. 
We highlight limiting properties of representations induced by (universally) consistent decision rules and demonstrate in simulation that an empirical estimate of task similarity is a function of the decision rule deployed for inference.
We demonstrate that for a given target distribution, both transfer efficiency and semantic similarity of candidate source distributions correlate with empirical task similarity.
Finally, we show that empirical task similarity captures an etymologically meaningful relationship between language tasks.  
\end{abstract}


%% file: text/introduction.tex
\section{Introduction} 
Recent successful applications of machine learning have come at the cost of incredible resource \citep{strubell2019energy, brown2020language} and data \citep{Krizhevsky09learningmultiple, balali2015detection, irvin2019chexpert} requirements. 
Even after making expensive models and data publicly available,
practitioners consistently face the problem of whether a particular pre-existing model or dataset can be gainfully used to assist them in building a successful model for their application. Indeed, the performance and generalization gap \citep{hand2006classifier} between research and production can be quite large.

Hence, it would be useful to be able to empirically evaluate the potential utility of a model or auxiliary dataset before engineering complex solutions based on them. To address this problem, we propose a similarity on a pair of classification tasks that is motivated by a fundamental object in statistical pattern recognition. 

We show that the empirical version of our proposed similarity positively correlates with accuracy improvements in a transfer setting.
In particular, we show in a CIFAR100-based setting that choosing the model pretrained on the source dataset that maximizes the proposed empirical similarity (amongst candidate source datasets) maximizes fine tuning accuracy on the target dataset.
Furthermore, we find that source-target pairs that share coarse labels have a higher task similarity source-target pairs that do not share coarse labels.

\noindent \textbf{Contributions:} We propose a similarity on classification distributions that is rooted in statistical pattern recognition, define relevant properties of pairs of distributions based on the proposed similarity, and highlight key analytical properties of the similarity for a broad class of well-studied decision functions. We demonstrate that the similarity can be used to assess the utility of auxiliary data in transfer learning settings.

%% file: text/background.tex
\section{Motivation \& Background}

\subsection{Classification Distribution Similarity Desiderata}
Recall that a classification distribution is a probability distribution defined on an input space $ \mc{X} $ and a categorical action space $ \mc{Y} = \{1, \hdots, k\} $. We call $ x \in \mc{X} $ a pattern and $ y \in \mc{Y} $ a label. In our discussion we care about a pair of classification distributions, $ F^{T} $ and $ F^{S} $ defined on $ \mc{X} \times \{1, \hdots, k^{T}\} $ and $ \mc{X} \times \{1, \hdots, k^{S}\} $, respectively.


There is a plethora of desirable properties of any similarity on $ F^{T} $ and $ F^{S} $. First, because a permutation of the labels corresponding to each of the class conditional distributions should not change the relationship between $ F^{T} $ and $ F^{S} $, a similarity on classification distributions should be invariant to a permutation of labels. Second, because a change in the categorical action space (either making it bigger or smaller) does not always affect a distribution in a meaningful way, a similarity on classification distributions should be well-defined on pairs of distributions with categorical action spaces of different sizes. Third, the maximal and minimal values of a similarity (if they exist) should be interpretable. That is, a similarity on classification distributions should be
\begin{enumerate}[label=\roman*)]
    \item invariant to a permutation of labels,
    \item well-defined for pairs of distributions with categorical action spaces of different sizes, and
    \item interpretable.
\end{enumerate}

\subsection{Statistical Pattern Recognition and Bayes rule}\label{subsec:stat-pat} Let 
\begin{align*} (X, Y), (X_{1}, Y_{1}), \ldots, (X_{n}, Y_{n}) \iid  F 
\end{align*}
be random variables distributed according to the joint distribution $ F $ with input realizations $ X_{i} = x_{i} \in \mc{X} $ and label realizations $ Y_{i} = y_{i} \in \{1, \hdots, k\} $. We let $ F_{X} $ be the marginal distribution of $ X_{i} $.  The objective in statistical pattern recognition \citep{devroye2013probabilistic, duda2012pattern} is to use the training data $ \mc{D}_{n} = \{(X_{i}, Y_{i})\}_{i \in \{1, \hdots, n\}} $ to learn a decision function $ h_{n} $ that correctly maps $ X $ to the true but unknown $ Y $. 

Formally, the objective in statistical pattern recognition is to minimize the risk of $ h_{n} $. The risk of a decision function $ h: \mc{X} \to \{1, \hdots, k\} $
is defined as the expected value of a loss function $ \ell: \{1, \hdots, k\} \times \{1, \hdots, k\} \to [0, \infty) $ evaluated at the output of the decision rule and the truth: $ \ell(h(X), Y) $. Or, the risk $ R $ of $ h $ for $ F $ is
\begin{align*}
    R_{F}(h) := \mbb{E}_{F}\left[\ell(h(X), Y)\right].
\end{align*} Typical loss functions include 0-1 loss, where the decision rule incurs a loss of $ 0 $ if $ h(X) = Y $ and $ 1 $ if $ h(X) \neq Y $ and cross-entropy. 


A decision function that minimizes $ R $ is known as \textit{Bayes rule} and is denoted $ h^{*} $. Under 0-1 loss $ h^{*} $ maps $ x \in \mc{X} $ to the class that maximizes the conditional probability after observing $ x $:
\begin{align*}
    h^{*}(x) = \argmax_{y \in \{1, \hdots, k\}} Pr(Y=y|X=x).
\end{align*} We say a sequence of decision functions, or a decision rule $ (h_{n})_{n=1}^{\infty} $, is \textit{consistent} for $ F $ if, as the number of training samples goes to infinity, the risk of $ h_{n} $ approaches $ R^{*}_{F} := R_{F}(h^{*}) $. A decision rule $ (h_{n})_{n=1}^{\infty} $ is \textit{universally consistent} if, as the number of training samples goes to infinity, the risk of $ h_{n} $ approaches $ R_{F}^{*} $ for all distributions $ F $. 

\label{subsec:transfer-learning}
\subsection{Transfer Learning} Transfer learning \citep{pan2009survey} is a generalization of classical statistical pattern recognition. Let
\begin{align*}
    (X_{1}, Y_{1}, T_{1}), \hdots, (X_{n}, Y_{n}, T_{n}) \iid F^{transfer}
\end{align*}
be random variables distributed according to the joint distribution $ F^{transfer} $ with realizations in $ \mc{X} \times \mc{Y} \times \mc{T} $ where $ \mc{X} $ and $ \mc{Y} $ are as in classical statistical pattern recognition and $ t \in \mc{T} = \{0,1\} $ is a task label. In particular, if $ t = 0 $ then $ (X, Y) \sim F^{S} $ and if $ t = 1 $ then $ (X, Y) \sim F^{T} $. We refer to $ F^{T} $ as the target distribution and $ F^{S} $ as the source distribution. The objective in transfer learning is to use the training data $ \mc{D}_{n} $ to learn a decision rule $ h_{n} $ that maps $ X $ to the true but unknown $ Y $ for patterns known to be from the target distribution.

A popular approach to transfer learning is to use the source data $ \mc{D}_{n}^{S} = \{(X_{i}, Y_{i}, T_{i}) \in \mc{D}_{n}: T_{i} = 0\} $ to learn a transformer $ u^{S}: \mc{X} \to \tilde{\mc{X}} $ that maps a pattern to a space $ \tilde{\mc{X}} $ more amenable for inference and to then use the target data $ D_{n}^{T} $ learn a classifier $ \tilde{h}^{T} $ from $ \tilde{\mc{X}} $ to the categorical action space \citep{duda2012pattern, thrun2012learning}. $ \tilde{\mc{X}} $ is sometimes referred to as a representation space \citep{bengio2012deep, bengio2013representation}. 
A more modular approach further decomposes $ \tilde{h}^{T} $ into two separate functions: a voter $ v^{T} $ and a decider $ w^{T} $. This decomposition enables omnidirectional transfer in diverse learning settings \citep{vogelstein2020general}.

We evaluate the efficacy of a transfer learning algorithm via transfer efficiency \citep{vogelstein2020general}:
\begin{align*}
TE(h_{n}^{S, T}, h_{n}^{T}) := \mbb{E}\left[{R_{F^{T}}\left(h_{n}^{S,T})\right)}\right]/\mbb{E}\left[{R_{F^{T}}\left(h_{n}^{T}\right)}\right]
\end{align*} where $ h_{n}^{S,T} $ has access to data from both the source distribution and the target distribution and $ h_{n}^{T} $ has access to data only from the target distribution. 




\subsection{Related work}

There has been a flurry of recent empirical work \citep{kifer2004detecting, achille2019task2vec, bao2018information, tran2019transferability, bhattacharjee2020p2l, nguyen2020leep} defining statistics that help to determine whether or not a particular model or representation learned on a set of source data will perform well when adjusted and applied to the target data. In the majority of these works \citep{bao2018information, tran2019transferability, nguyen2020leep} the authors begin their discussion with a pre-trained model on the source data and explore different ways to capture the utility of the model for a given target distribution.

We take a different perspective from the recent empirical work -- starting from the true-but-unknown distributions $ F^{T} $ and $ F^{S} $ -- and are thus able to both evaluate the proposed similarity empirically and study relevant properties analytically. Our approach, like LEEP \citep{nguyen2020leep}, does not require the patterns from the source and target data to be the same but is limited to patterns from the same input space.
One limitation of our work, as compared to others, is that an estimate of our similarity requires both a model trained on data from a source distribution and a model trained on data from the target distribution. 

From the theoretical side \citep{baxter2000model, ben2003exploiting, xue2007multi} our analysis is a mixture of Baxter's \citep{baxter2000model} and Xue et al's \citep{xue2007multi} in that we quantify the ``closeness" of the tasks via the representation of the data induced by the optimal decision function. 

Our work is also closely related to metrics and similarities defined on permutations of a finite set of objects, in particular the Rand index (and adjusted Rand index) \citep{doi:10.1080/01621459.1971.10482356} and ways to measure the similarity of different clusterings of the same objects \citep{hubert1985comparing}. Our proposed similarity is a generalization of these, as we both allow the patterns to be different and for them to be elements of a non-discrete space.

%% file: text/methods.tex
\section{Partitions and decision rules}

The similarity that we define in Section \ref{task-similarity-and-orthogonality} is based on partitions of $ \mc{X} $. We focus on partitions for two reasons: first, the optimal partition is a sufficient statistic for discrimination for a fixed loss function that retains interpretability in the original input space; second, many popular or successful modern machine learning algorithms, such as $K$-nearest neighbors \citep{1676031}, random forests \citep{breiman2001random} and neural networks with ReLU activation functions \citep{montufar2014number} partition $ \mc{X} $ and learn a posterior estimate per cell \citep{Priebe2020.04.29.068460}. Thus, partitions are both principled and practical.

For the remainder of this paper we assume that the loss function under consideration is 0-1 loss. \textbf{We note that generalizing to different losses may be non-trivial.}

Recall that a partition $ \mc{A} = \{A_{1}, \hdots, A_{m}\} $ of $ \mc{X} $ is a set of connected \textit{parts} or \textit{cells} such that $ \cup_{A \in \mc{A}} A = \mc{X} $ and $ A_{i} \cap A_{j} = \varnothing $ if $ i \neq j $. A Bayes rule $ h^{*} $ induces the optimal partition $ \mc{A}^{*} $ of $ \mc{X} $ for $ F $. An optimal partition is the smallest partition, as defined by the cardinality of $ \mc{A} $, such that the elements of $ A \in \mc{A}$ are all mapped to a single category by $ h^{*} $. We assume that the optimal partition is unique. The optimal partition $ \mc{A}^{*} $ is not unique to $ F $. For example, distributions that differ only in a permutation of the labels share an optimal partition. We let $ \mc{F}_{\mc{A}} $ be the set of distributions such that for all $ F \in \mc{F}_{\mc{A}} $, $ \mc{A} $ is its optimal partition. We say $\mc{B}$ is a subpartition of $\mc{A}$ if for all $A \in \mc{A}$, there exist a subset $\mc{B}_A \subset \mc{B}$ such that $A = \bigcup_{B \in \mc{B}_A}B$.


The partition $ \mc{A} $ induces a set of decision rules $ \mc{H}_{\mc{A}} $. This set is comprised of functions that label the collection of cells of the partition differently. In particular, if the number of labels is fixed to be $ k $ then $ |\mc{H}_{\mc{A}}| = k^{|\mc{A}|} $. We note that, by definition, the optimal decision function $ h^{*} $ for $ F $ is an element of the set of decision functions induced by its optimal partition $ \mc{A}^{*} = \{A^{*}_{1}, \hdots, A^{*}_{m}\} $. Furthermore, any decision function in $ \mc{H}_{\mc{A}} $, including the optimal decision function $ h^{*} $, is also an element of the set of decision functions induced by any subpartition of $ \mc{A} $. 

Let $(\mc{A}_n)_{n = 1}^{\infty}$ be a sequence of partitions. Let $R_{n} = \inf_{h \in \mc{H}_{\mc{A}_n}}R_{F}(h)$ be the best hypothesis for $ F $ in $ \mc{H}_{\mc{A}_{n}} $.
We say $(\mc{A}_n)_{n = 1}^{\infty}$ induces a consistent decision rule for $F$ if $ R_n \to R^*$ as $ n \to \infty $, where $R^*$ is the Bayes risk for $F$.
Similarly, we say a sequence of partitions $ (\mc{A}_{n})_{n=1}^{\infty} $ induces a universally consistent decision function if $ R_n \to R^*$ as $ n \to \infty $ for all $ F $. 

Herein we restrict our discussion to distributions such that $ \argmax_{y} Pr(Y=y | X=x) $ is unique almost everywhere and $ |\mc{A}^{*}| $ is countable. 

\section{Task Similarity and Orthogonality}\label{task-similarity-and-orthogonality}

Let $ F^{T} $ and $ F^{S} $ be two classification distributions defined on $ \mc{X} \times \{1, \hdots, k^{T}\} $ and $ \mc{X} \times \{1, \hdots, k^{S}\} $, respectively, with corresponding optimal decision functions $ h^{*}_{T} $ and $ h^{*}_{S} $ and optimal partitions $ \mc{A}^{*}_{T} $ and $ \mc{A}^{*}_{S} $.

In this section we define \textit{task similarity} as a measure of similarity between $ F^{T} $ and $ F^{S} $ based on their respective optimal partitions. Using the definition of task similarity, we then define \textit{adjusted task similarity} to ensure that the similarity has all three desiderata discussed above. Lastly, we use adjusted task similarity to define an \textit{adversarial distribution} and a pair of \textit{orthogonal distributions}.

The \textit{task similarity} of $ F^{T} $ and $ F^{S} $ is the measure of $ \mc{X} $, from the perspective of $ F^{T} $, in which an optimally transformed $ h^{*}_{S} $ agrees with $ h^{*}_{T} $. 

\begin{Def}[Task similarity]
Let $ F^{T} $ and $ F^{S} $ be two classification distributions with corresponding optimal decision functions $ h^{*}_T $ and $ h^{*}_S $ and optimal partitions $ \mc{A}_T^{*} $ and $ \mc{A}_S^{*} $, respectively. The task similarity of $ F^{S} $ to $ F^{T} $ is defined as
\begin{equation}\label{eq:task-similarity}
    TS(F^{T}, F^{S}) = 
    \sum_{A_{S} \in \mc{A}_S^*}
    \max_{y \in \{1, \dots, k^T\}}
    \int_{A_{S}}
    \mbb{I}\{h_{T}^{*}(x) = y\} dF^{T}_{X}.
\end{equation}
\end{Def}



The definition of task similarity can be thought of as a re-labeling of the cells of the optimal partition of the source distribution such that the new labeling agrees maximally (says the measure induced by $ F^{T} $) with $ h^{*}_{T} $. Indeed, task similarity is, loosely, a non-centered correlation of decision functions under a particular measure.

Task similarity \eqref{eq:task-similarity} satisfies desiderata i) and ii). It satisfies i) by virtue of the maximization and satisfies ii) by virtue of the maximization being over the set of labels corresponding to the target distribution. 
We note that $ TS(F^{T}, F^{S}) = 1 $ if and only if $ \mc{A}^{*}_{S} $ is a subpartition of $ \mc{A}^{*}_{T} $ and that $ TS(F^{T}, F^{S}) = TS(F^{S}, F^{T}) = 1 $ if and only if $ \mc{A}^{*}_{T} = \mc{A}^{*}_{S} $.
Hence, the maximal value of $ TS $ is interpretable. 

The minimal value of $ TS $, however, is not. As defined, task similarity does not account for the possibility that the argmax of the integrand in Equation \eqref{eq:task-similarity} may contain more than one element. This may happen, for example, in a two class classification problem when the conditional distributions of the target distribution have the same measure in a given cell of the optimal partition of $ F^{S} $. In that case, either of the two elements of the argmax would perform at chance on unlabeled patterns from $ F^{T} $.

To account for this (and with the goal of interpretability in mind) we remove the contributions of the mass from cells of $ \mc{A}^{*}_{S} $ where the argmax of the integrand of Equation \eqref{eq:task-similarity} is not unique. We define $ \psi_{A_{S}}^{T} $ to be the relevant argmax
\begin{align*}
    \psi_{A_{S}}^T = \argmax_{y \in \{1, \hdots, k^{T}\}} \int_{A_{S}}
    \mbb{I}\{h_{T}^{*}(x) = y\} dF^{T}_{X}
\end{align*} and only allow contributions to $ TS $ from cells of $ \mc{A}^{*}_{S} $ where the argmax is a singleton.

\begin{Def}[Adjusted Task Similarity]
Let $ F^{T} $ and $ F^{S} $ be two classification distributions with corresponding decision functions $ h^{*}_{T} $ and $ h^{*}_{S} $ and optimal partitions $ \mc{A}^{*}_{T} $ and $ \mc{A}^{*}_{S} $, respectively. Further, let $ \{\psi_{A_{S}}^{T}\}_{A_{S} \in \mc{A}_{S}} $ be the collection of argmaxes of the integrands in Equation \eqref{eq:task-similarity}. The adjusted task similarity of $ F^{S} $ to $ F^{T} $ is defined as
\begin{align}
    ATS(F^{T}, F^{S}) =
    \sum_{A_S \in \mc{A}_S^{*}} 
    \max_{y \in \{1, \dots, k^T\}}
    \int_{A_{S}} 
        \mbb{I}\{h^{*}_T(x) = y\}
        \cdot
        \mbb{I}\{|\psi_{A_{S}}^{T}| = 1\}
        dF^{T}_{X}.
\end{align}
\end{Def}
With this adjustment to task similarity, we have that $ ATS(F^{T}, F^{S}) = 0 $ if and only if the label that maximizes the integrand is not unique for all of $ \mc{X} $.  In balanced two class classification, this means that the decision rule that outputs a maximizer of the integrands performs at chance under 0-1 loss. Hence, the minimal and value of $ ATS(F^{T}, F^{S}) $ is interpretable and $ ATS $ satisfies desiderata iii).

Both task similarity and adjusted task similarity are inherently asymmetric, like the well studied $ f $-divergences \citep{csiszar1967information, liese2006divergences, cover2012elements}, because of the integration over $ \mc{X} $ with respect to the measure induced by $ F $. A symmetric similarity can be defined by considering a simple average of $ TS(F^{T}, F^{S}) $ and $ TS(F^{S}, F^{T}) $ or of $ ATS(F^{T}, F^{S}) $ and $ ATS(F^{T}, F^{S}) $. The symmetric $ ATS $ satisfies i), ii) and iii).

An interpretable lower bound allows us to define meaningful properties of the pair of distributions in terms of $ ATS $. In particular, a partition in which a distribution has a non-unique $ \argmax $ in every cell is maximally bad (in some sense). Any distribution that has this partition as its optimal partition can be thought of as ``adversarial" for that distribution. Following that line of thought, if both the distributions are adversarial for one another, these distributions can be thought of as non-informative, or ``orthogonal":

\begin{Def}[Adversarial Distribution]
A classification distribution $ F^{S} $ is \textit{adversarial} for $ F^{T} $ iff
\begin{equation*}\label{def:orthogonal-tasks}
    ATS(F^{T}, F^{S}) = 0
\end{equation*}
\end{Def}

\begin{Def}[Orthogonal Distributions]
Two classification distributions $ F^{S} $ and $ F^{T} $ are \textit{orthogonal} iff they are mutually adversarial.
\end{Def}

\subsection{Properties of (Adjusted) Task Similarity}\label{sec:properties}

We now formally state some properties of (adjusted) task similarity. The proofs of Theorems \ref{thm:ats=1=ats}, \ref{thm:non-decreasing}, \ref{thm:consistent_classifier} and \ref{thm:histogram-rules}, along with non-highlighted results, are in Appendix \ref{app:theorems}. 

\begin{theorem}\label{thm:ats=1=ats}
Let $\mc{A}$ be a partition on $\mc{X}$, and let $F^T, F^S \in \mc{F}_{\mc{A}}$, then $ATS(F^T, F^S) = ATS(F^S, F^T) = 1$.
\end{theorem}

This theorem confirms the intuition that the task similarity between distributions that have the same optimal partition should be $1$.

\begin{theorem}\label{thm:non-decreasing}
Let $ \mc{A} $ be a partition of $ \mc{X} $ and let $ \mc{B} $ be a subpartition of $ \mc{A} $. Let $ F_{\mc{A}} \in \mc{F}_{\mc{A}} $ and $ F_{\mc{B}} \in \mc{F}_{\mc{B}} $. Then for all $ F $
\begin{align*}
    TS(F, F_{\mc{A}}) \leq TS(F, F_{\mc{B}})
\end{align*}
\end{theorem}

This theorem confirms the observation that subpartitions are more expressive.

\begin{theorem}\label{thm:consistent_classifier}
Let $ F^{T} $ be a classification distribution and $ \mc{A}_{1}, \mc{A}_{2} , \hdots $ be a sequence of partitions that induces a consistent decision rule for $ F^{T} $. For a fixed $ n $ let $ F_n^{S} \in  \mc{F}_{\mc{A}_{n}} $. Then
\begin{equation*}
    \lim_{n \to \infty} TS(F^{T}, F^{S}_{n}) = 1.
\end{equation*}
\end{theorem}

In our setting, the closer optimal partitions get to one another (as defined by the risk of the optimal decision function in $ \mc{H}_{\mc{A}_{n}} $), the closer the distributions, and hence the higher the task similarity. If $\mc{A}$ is the partition induced by $F^T$, then in the limit  we can, loosely, think of the sequence of partitions yielding a subpartition of $\mc{A}$.

Let $ A \in \mc{A} $ for some partition $ \mc{A} $. We let $ diam(A) = \max ||x - x'|| $ for $ x, x' \in A $.
\begin{theorem}\label{thm:histogram-rules}
Let $ \mc{A}_{1}, \mc{A}_{2}, \hdots $ be a sequence of partitions. Suppose that $ \max_{A_{i_{n}}\in \mc{A}_{n}}  diam(A_{i_{n}}) \to 0 $ as $ n \to \infty $. For a fixed $ n $ let $ F^{S}_{n} \in  \mc{F}_{\mc{A}_{n}} $. Then for all $ F^{T} $
\begin{equation*}
    \lim_{n \to \infty} TS(F^{T}, F^{S}_{n}) = 1.
\end{equation*}

\end{theorem}

This is akin to a universal consistency result. Building on the previous theorem, here is a specific example of partitions which always yield a subpartition of the optimal partition for $ F^{T} $. The result is just but another perspective of ``histogram rules yield consistent decision rules" \citep{devroye2013probabilistic}[Chapter 6].

\subsection{Illustrative Examples}\label{subsec:illustrative}
We consider 4 distributions defined on $ [-1,1]^{2} $: XOR, Quadrants (Quads), Rotated-XOR (R-XOR), and Finer-XOR (F-XOR). Recall that XOR is the two class classification problem where, conditioned on being in class 0, a pattern is distributed uniformly on the positive (+, +) and negative quadrants (-, -) and where, conditioned on being in class 1, a pattern is distributed uniformly on the mixed quadrants (+, -) and (-, +). Quads is a four class distribution where classes are separated via the coordinate axes. Conditioned on being in a particular class a pattern is distributed uniformly in the quadrant. R-XOR is the same as XOR except the class conditional distributions are ``rotated" 45 degrees. F-XOR is the XOR problem in each of the four quadrants. Samples from these four distributions are shown in the top row of Figure \ref{fig:analytic}. 

The pairwise adjusted task similarities of these four distributions cover four cases that help develop an intuition, described in Appendix \ref{app:illustrative}, underlying (adjusted) task similarity and the definitions of adversarial and orthogonal tasks.


\begin{figure*}[!ht]
    \centering
    \includegraphics[width=\linewidth]{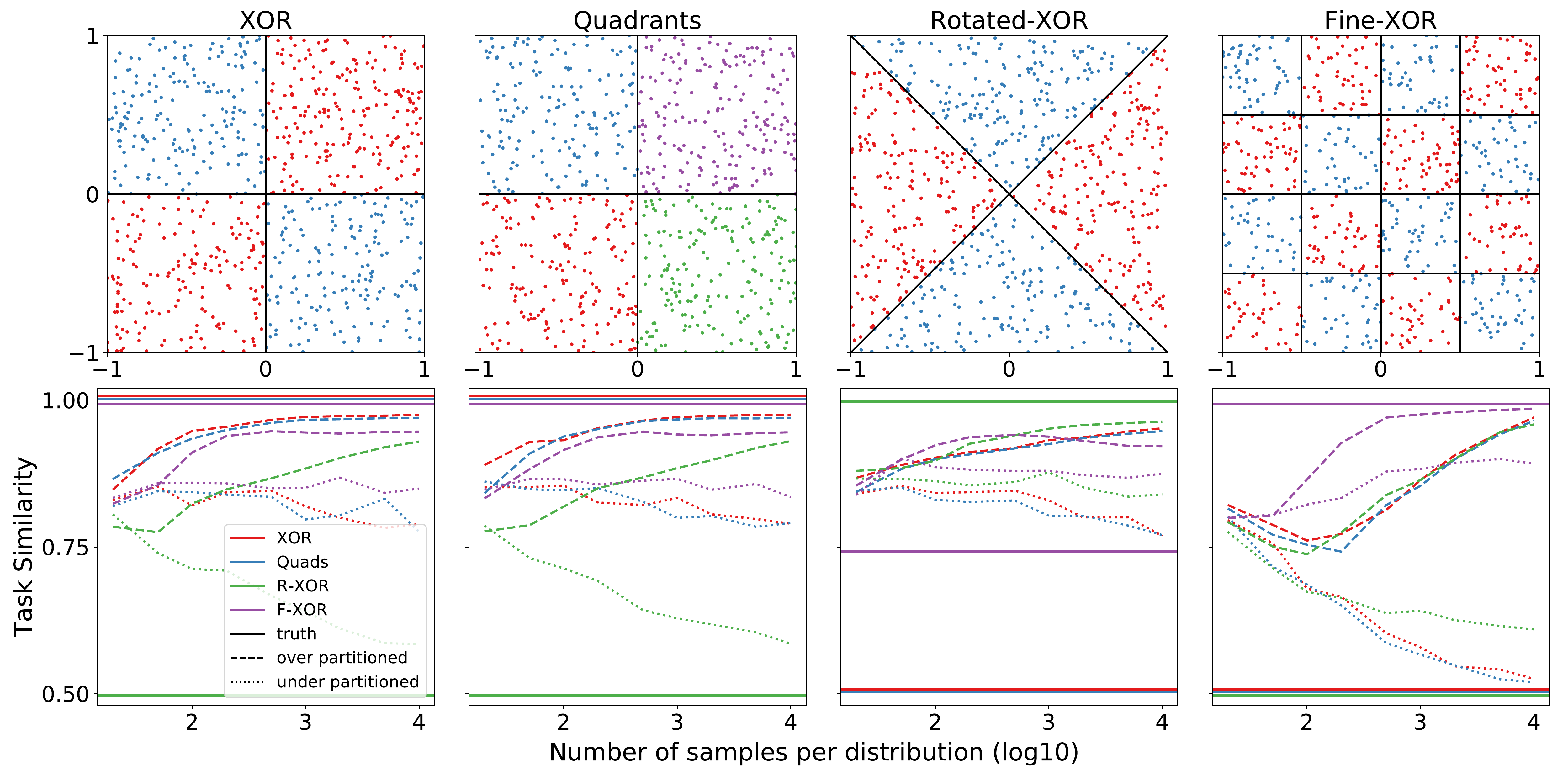}
    \caption{Pairwise analytic and empirical task similarities of XOR, Quads, R-XOR and F-XOR. 750 samples from each of the four distributions are shown in the top row. In particular, the column with XOR distribution in the top row corresponds to the (empirical and true) task similarity of XOR as the target task. Empirical task similarity was estimated using modified decision trees from \texttt{sci-kit learn} \citep{pedregosa2011scikit}.
    The task similarities of these four distributions demonstrate important properties of analytical and empirical task similarity (bottom row).
    }
    \label{fig:analytic}
 \end{figure*}

\section{Empirical Task Similarity}\label{sec:empirical-task-dissimilarity} Let
\begin{align*}
    (X_{1}, Y_{1}, T_{1}), \hdots, (X_{n}, Y_{n}, T_{n}) \iid F^{transfer}
\end{align*} be as described in Section \ref{subsec:transfer-learning}.

Task similarity, as defined by equation \eqref{eq:task-similarity}, cannot be measured directly because it depends on the true but unknown optimal partitions $ \mc{A}^{*} $ and $ \mc{B}^{*} $. To measure task similarity $ TS(F^{T}, F^{S}) $, then, we learn two composeable decision functions $ h^{S} = w^{S} \circ v^{S} \circ u^{S}  $ using $ \mc{D}_{n}^{S} $ and $ h^{T} = w^{T} \circ v^{T} \circ u^{T} $ using $ \mc{D}_{n}^{T} $. We let $ n^{T} $ denote the number of samples from the target distribution and $ w^{S,T} \circ v^{S,T} $ be the classifier learned using data from the target distribution after transforming it into $ \tilde{\mc{X}}^{S} $ using the transformer learned from the source data. Task similarity can then be measured by normalizing the number of agreements between $ w^{S,T} \circ v^{S,T} \circ u^{S} $ and $ w^{T} \circ v^{T} \circ u^{T} $ on $ \mc{D}^{T}_{n} $:
\begin{align}\label{eq:empirical task similarity}
    ETS(F^{T}, F^{S})
    = \frac{1}{n^{T}} \sum_{(X_{i}, Y_{i}, T_{i}) \in \mc{D}_{n}^{T}} \mbb{I}\{w^{T} \circ v^{T} \circ u^{T}(X_{i}) = w^{S,T} \circ v^{S,T} \circ u^{S}(X_{i})\}.
\end{align}

In general, $ ETS $ is \textbf{not} consistent for $ TS $. Indeed, for a fixed number of source samples $ n^{S} $, $ETS(F^{T}, F^{S})$ is consistent for $ TS(F^{T}, F^{S}_{n^{S}}) $ where $ F^{S}_{n^{S}} $ is in the set of distributions with optimal decision boundary the same as that induced by $  h^{S}_{n} $. Hence,
popular decision rules, such as $ K $ Nearest Neighbors \citep{fix1951discriminatory, stone1977consistent} with $ K \to \infty $ as a function of $ n $, decision trees \citep{amit1997shape} and forests \citep{breiman2001random} with a minimum depth that grows with $ n $, and artificial neural networks \citep{rosenblatt1958perceptron, hornik1991approximation} with a capacity that grows with $ n $ overpartition $ \mc{X} $ (or $ |\mc{A}_{n}| \gg |\mc{A}^{*}| $) and typically will not result in a consistent estimate of $ TS $. We demonstrate the effect of the choice of decision rule on empirical task similarity in the bottom row of Figure \ref{fig:analytic}. In general, when both algorithms overpartition the space the proposed estimate of task similarity tends towards 1. This corroborates Theorem \ref{thm:histogram-rules}. 

We note that the estimate of $ TS $ tending towards 1 does not directly translate to an effective partition (in terms of minimizing risk for the target distribution) for a given set of data. Indeed, overpartitioning can result in performance degradation when the true-but-unknown entropy in a cell is high relative to the amount of target data available.

Adjusted task similarity is inherently harder to estimate than task similarity, as it requires estimating the existence of different maximizers of the integrand in Equation \eqref{eq:task-similarity} per cell. We do not pursue any discussion of empirical adjusted task similarity.

\subsection{Assessing the utility of source distributions and pre-trained models}

\begin{figure}[t]
    \centering
    \includegraphics[width=.8\linewidth]{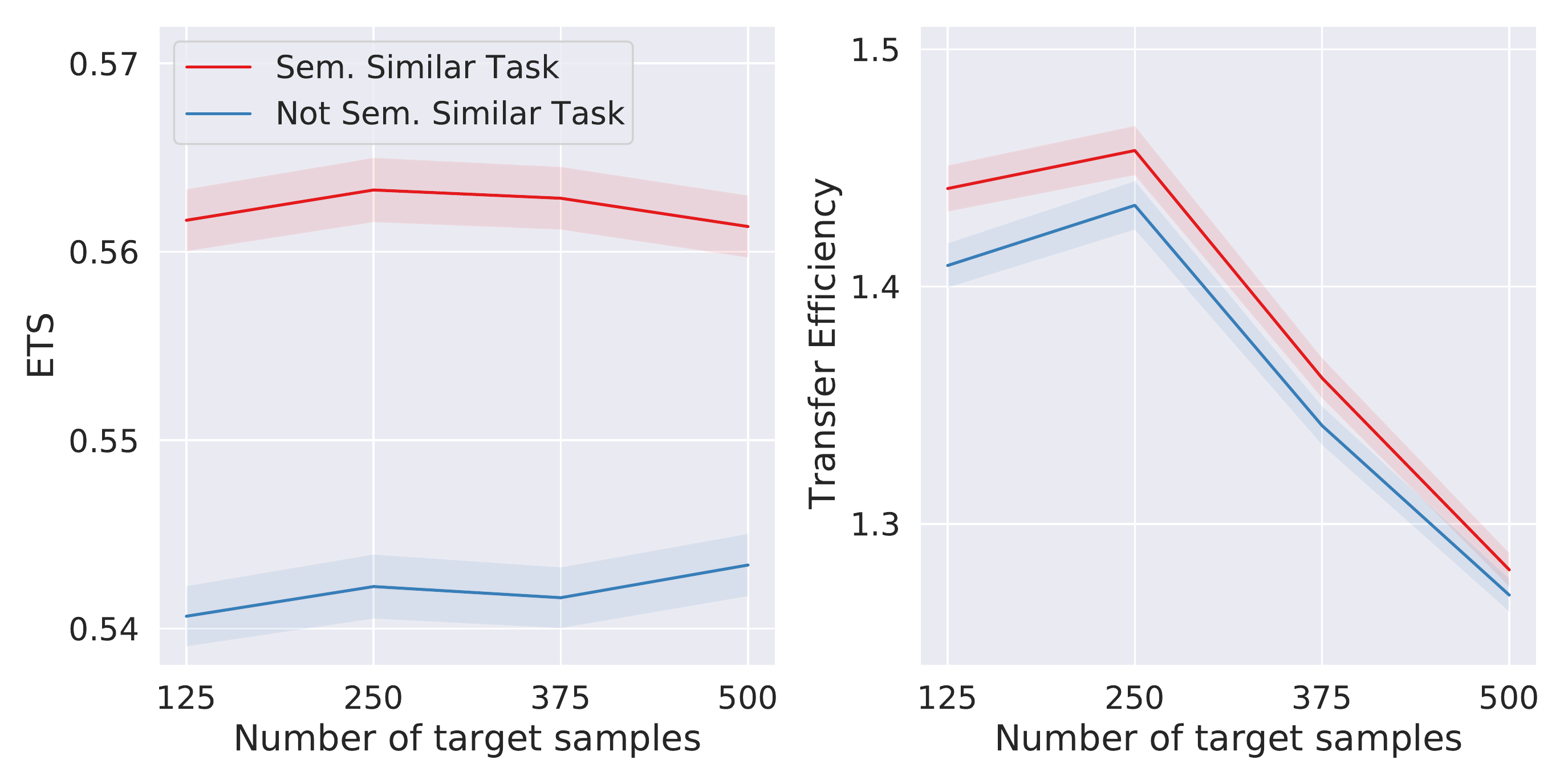}
    \caption{Empirical task similarity can be used to determine which pre-trained model to fine tune for a target task. In particular, semantically similar tasks both have a higher empirical task similarity (left) and pre-trained models trained on semantically similar tasks and subsequently fine tuned to the target task result in higher accuracy on the target task (right). Average baseline accuracies, i.e. training a model from scratch using data from the target distribution, used to measure transfer efficiency are $ 0.35, 0.42, 0.49 $ and $ 0.54 $ for $ n = 125, 250, 375 $ and $ 500 $, respectively.
    }
    \label{fig:cifar}
\end{figure}

In this section we demonstrate that empirical task similarity can be used to rank the efficacy of different models trained on different auxiliary datasets in an image classification setting and to identify etymologically similar languages in a language detection task. Assume that we have access to data from a target distribution, data from $ J $ different source (or auxiliary) distributions, and $ J $ composeable decision functions (one for each auxiliary distribution).
Our goal is to choose the auxiliary composeable decision function that, when adjusted to the target task, maximizes the accuracy on the target task. 

\subsubsection{CIFAR100}\label{sec:CIFAR}

We first investigate this problem in the context of image classification (CIFAR100 \citep{Krizhevsky09learningmultiple}).
Recall that for the CIFAR100 dataset there exists two labels, a coarse label and a fine label, for each input pattern.
In our setting, we construct a target distribution then subsequently construct two candidate source distributions: one that is ``semantically similar" to the target distribution, one that is ``semantically dissimilar" to the target distribution. 
All distributions contain five non-overlapping classes.

To construct the target distribution we first randomly select five coarse labels.
We then randomly select a fine label from the set of fine labels corresponding to each of the coarse labels (i.e. the lion fine label from the large carnivore coarse label).
The source distribution that is ``semantically similar" to the target task is then constructed by sampling (without replacement) a single distribution from each of the sets corresponding to the coarse labels sampled for the target distribution.
That is, for example, the target distribution and the ``semantically similar" auxiliary distribution both contain a class from within the large carnivore, small mammal, flowers, people and reptile coarse labels.
The source distribution that is ``semantically dissimilar" to the target task is constructed by randomly sampling fine labels from coarse labels not sampled for the target task.

Figure \ref{fig:cifar} shows the results of a Monte Carlo simulation used to evaluate the task similarity (left) and transfer efficiency of pretraining on the various source distributions (right).
At every replication of the simulation, new target and source distributions are randomly generated from CIFAR100.
Then a small convolutional neural network is trained on each of the source datasets, as well as on the target dataset to establish an accuracy baseline (see Appendix \ref{app:network} for architecture and training details).
Finally, each of the convolutional networks trained on the source datasets is fine-tuned on the target dataset.
The task similarity and transfer efficiency resulting from pretraining are recorded and shown in Figure \ref{fig:cifar}.
The error bars on the figure denote 90\% confidence intervals for the mean.

Note that the ``semantically similar" source distribution has significantly higher task similarity than the ``semantically dissimilar" source distribution for all numbers of target samples, and significantly better transfer efficiency for all shown amounts of target samples except $ n = 500$. This result indicates that empirical task similarity can be used to select a pre-trained model to use for transfer.

\subsubsection{Language detection}

For the task of language detection, i.e. determining the language of a sentence, we trained an unsupervised skipgram subword embedding on the Tatoeba multi-language dataset \citep{TIEDEMANN12.463} using an implementation from \texttt{fasttext} \citep{bojanowski2017enriching}. Sentence embeddings were generated by first tokenizing words into tokens, normalizing each token for each sentence, and taking the sum of the normalized tokens. For purposes of measuring empirical task similarity, we sampled 10,000 sentences from each of twenty-nine languages. If a language had less than 10,000 sentences we included the entire collection.  

We then selected four pairs of languages of interest, totaling eight languages. Each pair are ``etymologically" similar --  for example, Portuguese and Spanish share a recent common parent language. Once the pairs were selected, we then combined the remaining twenty-one languages into a ``negative" language class. For each language in at least one selected pair, we constructed a two-class classification task by pairing each of the languages with the negative class. We then measured the pairwise empirical task similarity between each of these classification tasks.

Figure \ref{fig:lang} shows the pairwise empirical task similarity between each of the selected languages. For all proportion of samples from the target task, the ranking of the empirical task similarities is etymologically meaningful -- i.e. Japanese is closest to Chinese Mandarin (and vice versa), Portugese is closest to Spanish (and vice versa), etc.

\begin{figure}
    \centering
    \includegraphics[width=\textwidth]{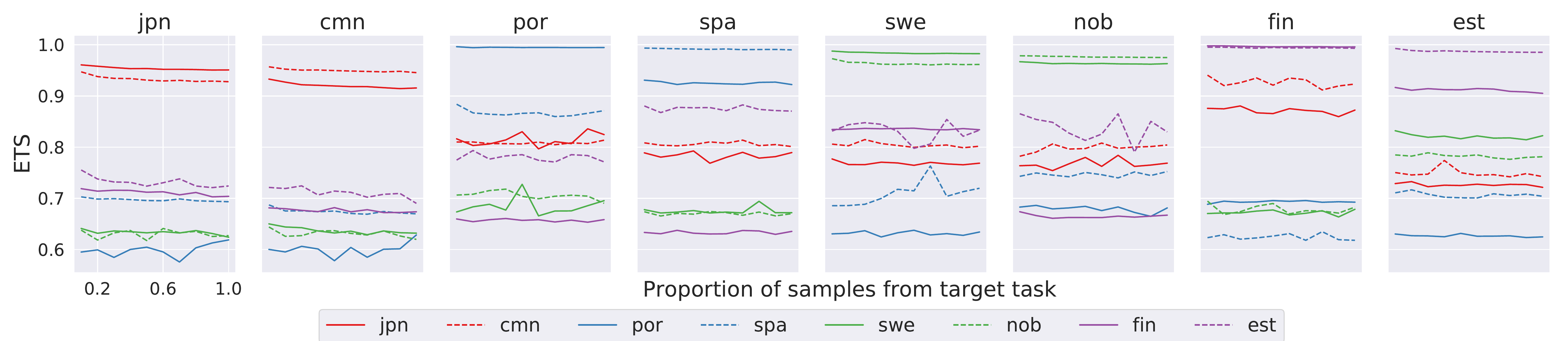}
    \caption{Empirical task similarity (ETS) of different languages as measured through the lens of a language detection task. ETS appears to be a useful proxy for language similarity.}
    \label{fig:lang}
\end{figure}

%% file: text/discussion.tex
The definition of task similarity is principled from the perspective of statistical pattern recognition. Further, it being a function of optimal partitions ties it closely to popular modern machine learning algorithms \citep{Priebe2020.04.29.068460}. Indeed, the analysis of Section \ref{sec:properties} elucidates the utility of algorithms such as deep neural networks, random forests, and $ K$-nearest neighbors that learn representations of the input space that are more than sufficient for learning an optimal classifier for a given distribution.

Empirical task similarity may be useful from the perspective of machine learning practitioners. 
For a particular task induced by the CIFAR-100 coarse labels, the ranking the empirical task similarity of different auxilliary data correlates with the ranking of the accuracy of transfer models corresponding to the different auxiliary datasets. Hence, we have proposed a promising solution to the problem of identifying useful datasets and models before training begins for the target task. 

Extensions of this similarity to distributions defined on different input spaces (i.e. $ \mc{X}^{S} = \mbb{R}^{d} $ and $ \mc{X}^{T} = \mbb{R}^{d'} $) is an important and natural next research objective and likely requires including an optimization over a non-discrete set of projections from $ \mbb{R}^{d} $ to $ \mbb{R}^{d'} $.

Another extension of the work herein is to formally describe the set of similarities and dissimilarities on the set of classification distributions that are defined similarly to TS. In particular, the integrand of Equation \eqref{eq:task-similarity} is but one way to measure sameness. In this direction, it seems gainful to investigate the relationship between $ f $-divergences and the proposed task similarity.

Finally, we note that theoretical notions of similarity can likely be extended to different inference tasks including regression and ranking.







%% file: text/appendix/theorems.tex
\section{Proofs of Theorems  \ref{thm:ats=1=ats}, \ref{thm:non-decreasing}, \ref{thm:consistent_classifier} and \ref{thm:histogram-rules}}\label{app:theorems}

\subsection*{Theorem \ref{thm:ats=1=ats}}
\begin{proof}
    Assume $F^T, F^S \in \mc{F}_{A}$. Thus $\mc{A}$ is the optimal partition of both $F^T$ and $F^S$. 
    
    Let $F^T$ have $k^T$ classes and let $F^S$ have $k^S$ classes. We have by definition of adjusted task similarity 
    \begin{align*}
        ATS(F^T, F^S) = \sum_{A \in \mc{A}}\max_{y \in \{1, \dots, k^T\}}\int_{A} 
        \mbb{I}\{h^{*}_T(x) = y\}
        \cdot
        \mbb{I}\{|\psi_{A}^{T}| = 1\}
        dF^{T}_{X}.
    \end{align*}
    Examining the term $\mbb{I}\{|\psi_{A}^{T}| = 1\}$,
    \begin{align*}
         \mbb{I}\{|\psi_{A}^{T}| = 1\} = \mbb{I}\{|\argmax_{y \in \{1, \hdots, k^{T}\}} \int_{A}
    \mbb{I}\{h_{T}^{*}(x) = y\} dF^{T}_{X}| = 1\}
    \end{align*}
    Since $\mc{A}$ is the optimal partition for $F^T$, we have that $h_T^*$ (the optimal hypothesis) is constant on $A \in \mc{A}$, say $h_T^* = y_A$. The integral above by is maximized by taking $y$ to be $y_A$, as any other $y$ would make it zero. Thus the argmax always has only one element, $y_A$, and the indicator is always equal to one. Thus 
    \begin{align*}
        ATS(F^T, F^S) & = \sum_{A \in \mc{A}}\max_{y \in \{1, \dots, k^T
        \}}\int_{A} 
        \mbb{I}\{h^{*}_T(x) = y\}
        \cdot
        \mbb{I}\{|\psi_{A}^{T}| = 1\}
        dF^{T}_{X}\\
        &= \sum_{A \in \mc{A}}\max_{y \in \{1, \dots, k^T\}}\int_{A} \mbb{I}\{h^{*}_T(x) = y\} dF^{T}_{X}
    \end{align*}
    By the same argument again, as $\mc{A}$ is the optimal partition for $F^T$ and hence $h^*_T$ is constant on $A$, we get that the indicator is always $1$ as $y$ is taken to be $y_A$, which is the constant value that $h_T^*$ takes on in $A$. Thus 
     \begin{align*}
        TS(F^T, F^S) & = 
        \sum_{A \in \mc{A}}\max_{y \in \{1, \dots, k^T\}}\int_{A} \mbb{I}\{h^{*}_T(x) = y\} dF^{T}_{X} \\
        & = \sum_{A \in \mc{A}}\int_{A}dF^T_X \\
        & = 1
    \end{align*}
    Since $\mc{A}$ is a partition for $\mc{X}$. Thus $ATS(F^T, F^S) = 1$. By symmetry, $ATS(F^S, F^T) = 1$ also, and so the desired result.
\end{proof}

\subsection*{Theorem \ref{thm:non-decreasing}}

\begin{proof}
   Let $\mc{B}$ be a subpartition of $\mc{A}$ and let $F_{\mc{A}} \in \mc{F}_{\mc{A}}$ and $F_{\mc{B}} \in \mc{F}_{\mc{B}}$. Let $F^T$ be an arbitrary distribution on $\mc{X}\times\{1, \dots, k\}$. 
   
   Consider now $\Delta = TS(F^T, F_{\mc{B}}) - TS(F^T, F_{\mc{A}})$. We would like to show this quantity is greater than or equal to zero, and we get the desired result.
   
   Since $\mc{B}$ is a subpartition of $\mc{A}$, then every element of $\mc{A}$ is the union of the elements of some subset of $\mc{B}$. Thus let $A \in \mc{A}$ be arbitrary, and let $\mc{B}_{A}$ be such that $A = \bigcup_{B \in \mc{B}_A}B$. Then,
   \begin{align*}
    \sum_{B \in \mc{B}_A}
    \max_{y \in \{1, \dots, k^T\}}
    \int_{B} 
        \mbb{I}\{h^{*}_T(x) = y\}
        dF^{T}_{X}
        -
    \max_{y \in \{1, \dots, k^T\}}
    \int_{A} 
        \mbb{I}\{h^{*}_T(x) = y\}
        dF^{T}_{X}.
  = \\      
   \sum_{B \in \mc{B}_A}
    \max_{y \in \{1, \dots, k^T\}}
    \int_{B} 
        \mbb{I}\{h^{*}_T(x) = y\}
        dF^{T}_{X}
        -      
    \max_{y \in \{1, \dots, k^T\}}
    \sum_{B \in \mc{B}_A}
    \int_{B} 
        \mbb{I}\{h^{*}_T(x) = y\}
        dF^{T}_{X}
   \end{align*}
   Hence the difference between $TS(F^T, F_\mc{B})$ and $TS(F^T, F_{\mc{A}})$ is simply moving the maximum inside the sum, i.e. maximizing each integral separately rather than maximizing them all simultaneously. Now unless the same $y$ maximizes each integral, in which case the expressions are equal, maximizing the sum yields a quantity smaller than maximizing each summand separately then summing. Hence,
   \begin{align*}
        \sum_{B \in \mc{B}_A}
    \max_{y \in \{1, \dots, k^T\}}
    \int_{B} 
        \mbb{I}\{h^{*}_T(x) = y\}
        dF^{T}_{X} 
        - 
        \max_{y \in \{1, \dots, k^T\}}
    \sum_{B \in \mc{B}_A}
    \int_{B} 
        \mbb{I}\{h^{*}_T(x) = y\}
        dF^{T}_{X} \geq 0
   \end{align*}
   Now summing over each $A \in \mc{A}$,
    \begin{align*}
        \sum_{A \in \mc{A}}\sum_{B \in \mc{B}_A}
    \max_{y \in \{1, \dots, k^T\}}
    \int_{B} 
        \mbb{I}\{h^{*}_T(x) = y\}
        dF^{T}_{X} 
        -
        \sum_{A \in \mc{A}}\max_{y \in \{1, \dots, k^T\}}
    \sum_{B \in \mc{B}_A}
    \int_{B} 
        \mbb{I}\{h^{*}_T(x) = y\}
        dF^{T}_{X} \geq 0
   \end{align*}
   Since $\mc{B}$ subpartition of $\mc{A}$ implies that $\mc{B} = \bigcup_{A \in \mc{A}}\mc{B}_{\mc{A}}$. Hence we get
   \begin{align*}
     TS(F^T, F_{\mc{B}}) - TS(F^T, F_{\mc{A}}) \geq 0
   \end{align*}
   Which is the desired result.
\end{proof}

\subsection*{Theorem \ref{thm:consistent_classifier}}

\begin{proof}
    Assume the conditions in the theorem. Consider the task similarity $TS(F^T, F^S_n)$,
    \begin{align*}
         TS(F^{T}, F^{S}) = 
    \sum_{A_{S} \in \mc{A}_S^*}
    \max_{y \in \{1, \dots, k^T\}}
    \int_{A_{S}}
    \mbb{I}\{h_{T}^{*}(x) = y\} dF^{T}_{X}.
    \end{align*}
    In the above integral, we pick some $y = y_{A_S}$ such that $y_{A_S}$ maximizes $\int_{A_{S}}\mbb{I}\{h_{T}^{*}(x) = y\} dF^{T}_{X}$. In essence we have just defined $h_n$ where $h_n(x) = y_{A_S}$ whenever $x \in A_S$.In other words, $h_n$ takes on the constant value $y_{A_S}$ on each cell $A_S \in \mc{A}_S$. The task similarity then turns into
    \begin{align*}
         TS(F^{T}, F^{S}) & = 
    \sum_{A_{S} \in \mc{A}_S^*}
    \max_{y \in \{1, \dots, k^T\}}
    \int_{A_{S}}
    \mbb{I}\{h_{T}^{*}(x) = y\} dF^{T}_{X} \\
    & = \sum_{A_S \in \mc{A}_S}\int_{A_S}\mbb{I}(h_T^*(x) = h_n(x))dF^T_X \\
    & = \int_{\mc{X}}\mbb{I}\{h_T^*(x) = h_n(x)\}dF^T_X
    \end{align*}
    Let $R_n = R_{F^T}(h_n)$ be the risk of $h_n$ and $R^* = R^*_{F^T}$ be the Bayes risk for $F^T$. $\mc{A}_1, \mc{A}_2, \dots$ inducing a consistent decision rule implies that
    \begin{align*}
       R_n \to R^*
    \end{align*}
    Or (recall we are assuming 0-1 loss)
    \begin{align*}
        \int_{\mc{X}}\sum_{y \in \{1, \dots, k^T\}}\mbb{I}(h_n(x) \neq y)p(y)dF_X^T \to \int_{\mc{X}}\sum_{y \in \{1, \dots, k^T\}}\mbb{I}(h^*_T(x) \neq y)p(y)dF_X^T
    \end{align*}
    as $n \to \infty$. Without loss of generality, we assume all $x \in \mc{X}$ have positive density/probability (simply throw out the points with zero probability/density). 
    
    Since $R_n \to R^*$, and $R^*$ is the infimum of the risk of all decision rules, and we are assuming $p(y|x)$ is uniquely maximized by some $y$, and all the points $x \in \mc{X}$ have positive density, then we must have that $h_{n} \to h_{T}^{*}$ pointwise. 
    
    First, if $\lim h_n(x)$ does not exist, then $\lim R_n$ does not exist, which contradicts our assumption. Thus $\lim h_n(x)$ exists for all $x$.
    
    Next, if for any point $x$ we have that $\lim_{n\to\infty}h_n(x) \neq h_T^*(x)$, then $\lim_{n\to\infty}R_n > R^{*} $. This is because $x$ has positive density, and if $\lim h_n \neq h_T^*$, then by uniqueness of the $y$ that maximizes $p(y|x)$ and $h^*_T$ being the Bayes decision rule (hence lowest risk), we get a risk that is bigger than the Bayes risk in the limit, contradicting our assumption that $R_n \to R^*$. Thus $h_n \to h_T^*$ pointwise as claimed.
    
    Going back to task similarity, we thus have that
    \begin{align*}
        \lim_{n\to\infty} TS(F^T, F^S_n) & = \lim_{n\to\infty}\int_{\mc{X}}\mbb{I}(h_T^*(x) = h_n(x)\}dF^T_X \\
        & = \int_{\mc{X}}\lim_{n\to\infty}\mbb{I}(h_T^*(x) = h_n(x)\}dF^T_X \\
        & = \int_{\mc{X}}\mbb{I}(h_T^*(x) = h_T^*(x)\}dF^T_X \\
        & = \int_{\mc{X}}dF^T_X \\
        & = 1
    \end{align*}
    Where the interchange of the limit and the integral is justified by the dominated convergence theorem. Thus,
    \begin{align*}
         \lim_{n\to\infty} TS(F^T, F^S_n) = 1
    \end{align*}
    as desired.
\end{proof}

\subsection*{Theorem \ref{thm:histogram-rules}}
\begin{proof}
    Let $ \mc{A}_{1}, \mc{A}_{2}, \hdots $ be a sequence of partitions such that $ \max_{A_{i_{n}}\in \mc{A}_{n}}  diam(A_{i_{n}}) \to 0 $ as $ n \to \infty $. Then any sequence any $(h_n)_{n = 1}^{\infty}$ such that $h_n \in \mc{H}_{\mc{A}_n}$ is a histogram rule. Let $F^T$ be an arbitrary distribution on $\mc{X}\times\{1, \dots, k\}$. Since histogram rules are universally consistent, we have that there must exist a specific sequence $(h_n)_{n = 1}^{\infty}$ such that $R(h_n) \to R^*$, a $n \to \infty$, where $R^*$ is the Bayes risk associated with $F^T$. Thus $(\mc{A}_n)_{n = 1}^{\infty}$ induces a consistent decision rule for $F^T$. By theorem \ref{thm:consistent_classifier}, $\lim_{n \to \infty}TS(F^T, F^S_n) = 1$ whenever $F^S_n \in \mc{F}_{\mc{A}_n}$, as desired.
\end{proof}

\begin{theorem}
$ATS(F, G) \le TS(F,G)$ for all $ F, G $ where $ F, G $ are classification distributions.
\end{theorem}

\begin{proof}
   Let $F^S, F^T$ be arbitrary distributions respectively on $\mc{X}\times\{1, \dots, k^S\}$ and $\mc{X}\times\{1, \dots, k^T\}$. Consider $ATS(F^T, F^S)$ and $TS(F^T, F^S)$. Since
   \begin{align*}
       \int_{A_S} \mbb{I}\{h^{*}_T(x) = y\}\mbb{I}\{|\psi_{A_S}^{T}| = 1\} dF^{T}_{X} \geq \int_{A_S} \mbb{I}\{h^{*}_T(x) = y\}dF^{T}_{X} 
   \end{align*}
   For all $A_S \in \mc{A}_S$ where $\mc{A}_S$ is the optimal partition induced by $F^S$, we have that by summing over $A_S$, we get $TS(F^T,F^S) \geq ATS(F^T, F^S)$, as desired.
\end{proof}

\begin{theorem}
Let $\mc{A}$ be a partition on $\mc{X}$, and let $F^T, F^S \in \mc{F}_{\mc{A}}$, then $TS(F^T, F^S) = TS(F^S, F^T) = 1$.
\end{theorem}

\begin{proof}
   Assume $F^T, F^S \in \mc{F}_{A}$. Thus $\mc{A}$ is the optimal partition of both $F^T$ and $F^S$. 
    
    Let $F^T$ have $k^T$ classes and let $F^S$ have $k^S$ classes. We have by definition of task similarity 
    \begin{align*}
        TS(F^T, F^S) = \sum_{A \in \mc{A}}\max_{y \in \{1, \dots, k^T\}}\int_{A} 
        \mbb{I}\{h^{*}_T(x) = y\}
        dF^{T}_{X}.
    \end{align*}
    Since $\mc{A}$ is the optimal partition for $F^T$, we have that $h_T^*$ is constant on $A \in \mc{A}$, say $h_T^* = y_A$. The integral above by is maximized by taking $y$ to be $y_A$, as any other $y$ would make it zero. Thus the indicator is always equal to one, and so
     \begin{align*}
        TS(F^T, F^S) & = 
        \sum_{A \in \mc{A}}\max_{y \in \{1, \dots, k^T\}}\int_{A} \mbb{I}\{h^{*}_T(x) = y\} dF^{T}_{X} \\
        & = \sum_{A \in \mc{A}}\int_{A}dF^T_X \\
        & = 1
    \end{align*}
    Since $\mc{A}$ is a partition for $\mc{X}$. Thus $TS(F^T, F^S) = 1$. By symmetry, $TS(F^S, F^T) = 1$ also, and so the desired result. 
\end{proof}

%% file: text/appendix/illustrative.tex
\section{Description of pairwise properties of XOR, Quads, R-XOR and F-XOR}\label{app:illustrative}

The four distributions that we study in Section \ref{subsec:illustrative} have illustrative pairwise adjusted task similarities. We highlight four of these properties.

1. XOR and Quads share an optimal partition $ \mc{A}^{*} $, implying that $ ATS(\text{XOR, Quads}) = ATS(\text{Quads, XOR}) = 1 $ by Theorem \ref{thm:ats=1=ats}.

2. XOR and R-XOR are orthogonal. Letting XOR be the target task, the optimal partitions induced by R-XOR are such that $ y = 0 $ and $ y = 1 $ maximize the integrand in Equation \eqref{eq:task-similarity}. Hence the argmax is not unique for all of $ \mc{X} $ and $ ATS(\text{XOR, R-XOR}) = 0 $. The same is true for $ ATS(\text{R-XOR, XOR}) $.

3. XOR is adversarial for F-XOR and F-XOR is maximally similar to XOR. Focusing on $ ATS(\text{F-XOR, XOR}) $, for each quadrant both $ y = 0 $ and $ y = 1 $ maximize the integrand in Equation \eqref{eq:task-similarity}. Hence the argmax is not unique for all of $ \mc{X} $ and $ ATS(\text{F-XOR, XOR}) = 0 $. On the other side, setting $ y = 0 $ or $ y = 1 $ for each quadrant within a quadrant maximizes the integrand in Equation \eqref{eq:task-similarity} and, moreover, is equal to the optimal decision rule for XOR. Hence, $ ATS(\text{XOR, F-XOR}) = 1 $.

4. R-XOR is adversarial for F-XOR and $ ATS(\text{R-XOR, F-XOR}) $ is neither 0 nor 1.  Again, for each cell in the optimal partition induced by R-XOR, both $ y = 0 $ and $ y = 1 $ optimize the integrand of Equation \eqref{eq:task-similarity} when applied to F-XOR, implying $ ATS(\text{F-XOR, R-XOR}) = 0 $. When applying the optimal partition induced by F-XOR to R-XOR, the cells that do not contain either the lines $ x_{0} = x_{1} $ or $ x_{0} = -x_{1} $ are uniquely maximized by either $ y = 0 $ or $ y = 1 $, meaning $ ATS(\text{F-XOR, R-XOR}) > 0 $. On the other side, the cells of the optimal partition of F-XOR that do contain either of those lines are such that $ y = 0 $ and $ y = 1 $ both maximize the integrand of Equation \eqref{eq:task-similarity}. Hence, $ ATS(\text{R-XOR, F-XOR}) < 1 $. 




%% file: text/appendix/experiments.tex


%% file: text/appendix/network.tex
\section{Network Training Details}\label{app:network}
All networks in the experiments of section \ref{sec:CIFAR} were trained with the same architecture, as detailed in the pseudocode below:
\begin{verbatim}
    x = Conv2D(filters=8,
               kernel_size=(3, 3),
               strides=(2, 2))(input_layer)
   
    x = Conv2D(filters=32,
               kernel_size=(3, 3),
               strides=(2, 2))(x)
   
    x = Flatten()(x)
    x = BatchNormalization()(x)
    rep = Dense(1024, activation=relu)(x)
    x = Dropout(.5)(rep)
    y_hat = Dense(5, activation=softmax)(x)
\end{verbatim}
These networks were trained for 20 epochs with the Adam optimizer, with a learning rate of 1e-4, a $beta_1$ of 0.9, a $beta_2$ of 0.999 and an $\epsilon$ of 1e-7, using a categorical crossentropy objective.